\documentclass[conference]{IEEEtran}

\usepackage[T1]{fontenc}
\usepackage[cmintegrals]{newtxmath}
\usepackage{newtxtext}

\usepackage{fancyhdr}
\AtBeginDocument{
    \setlength{\columnsep}{0.25in}
}

\IEEEoverridecommandlockouts

\usepackage{amsmath}
\usepackage{subeqnarray}
\usepackage{cases}
\usepackage{cite}
\usepackage{graphicx}
\usepackage{tabularx,booktabs}
\usepackage{diagbox}
\usepackage[ruled,linesnumbered]{algorithm2e}
\usepackage{algorithmic}
\usepackage{esint}
\usepackage{color}
\usepackage{lipsum}
\usepackage{cuted}
\usepackage{stfloats}
\usepackage{multirow}
\usepackage{ntheorem}
\usepackage{color}
\usepackage{threeparttable}
\usepackage{bbm}
\usepackage{tabu}
\usepackage{bbding}
\usepackage{tikz}
\usepackage{xcolor}
\usepackage{makecell}
\usepackage{subfigure}
\usepackage{subfig}

\theoremseparator{:}
\newtheorem{proposition}{\textcolor{black}{Proposition}}
\newtheorem{lemma}{Lemma}
\newtheorem{theorem}{Theorem}

\theoremstyle{nonumberplain}
\theorembodyfont{\normalfont}
\newtheorem{proof}{Proof}
\newtheorem{Proof}{Proof}

\newcommand{\posgt}[1]{\mathbf{Y}_{#1}}
\newcommand{\posest}[1]{\hat{\mathbf{Y}}_{#1}}

\definecolor{newextractedpurple}{RGB}{127,0,127}

\def\BibTeX{{\rm B\kern-.05em{\sc i\kern-.025em b}\kern-.08em
    T\kern-.1667em\lower.7ex\hbox{E}\kern-.125emX}}

\begin{document}
\thispagestyle{empty}

\setlength{\columnsep}{0.25in}

\title{Task-Oriented Semantic Compression for Localization at the Network Edge}

\author{
    \IEEEauthorblockN{Zhengru Fang${}^{\star}$, Senkang Hu${}^{\star}$, Yu Guo${}^{\star}$, 
Yiqin Deng${}^{\star}$, Yuguang Fang${}^{\star}$}
    \IEEEauthorblockA{$^\star$Hong Kong JC STEM Lab of Smart City and Department of Computer Science,\\
    City University of Hong Kong, Hong Kong.\\
    Email: \{zhefang4-c, senkang.forest\}@my.cityu.edu.hk,\{yu.guo, yiqideng, my.fang\}@cityu.edu.hk}
    \vspace{-8mm}
}

\maketitle

\begin{abstract}
Achieving precise visual localization in GPS-limited urban environments poses significant challenges for resource-constrained mobile platforms, particularly under strict bandwidth, memory, and processing limitations. Inspired by mammalian spatial cognition, we propose a task-oriented communication framework in which bandwidth-limited endpoints equipped with multi-camera systems extract compact multi-view features and offload localization tasks to collaborative edge servers. We introduce the {\underline{O}}rthogonally-constrained {\underline{V}}ariational {\underline{I}}nformation {\underline{B}}ottleneck encoder (O-VIB), which incorporates automatic relevance determination (ARD) to prune non-informative features while enforcing orthogonality to minimize redundancy. This enables efficient and accurate localization with minimal transmission overhead. Extensive evaluation on a real-world urban localization dataset demonstrates that O-VIB achieves high-precision localization under stringent bandwidth budgets, outperforming existing methods across diverse communication constraints. 
\end{abstract}


\section{Introduction}
The rapid proliferation of resource-constrained mobile platforms in applications such as last-mile parcel delivery, food and grocery logistics, emergency medical supply transport, time-sensitive pharmaceutical delivery, and urban infrastructure inspection demands reliable and precise positioning \cite{10536071,10937373}. While traditional radio-based positioning remains a common method, its susceptibility to signal degradation and multi-path interference significantly undermines reliability in dense urban canyon environments \cite{sathaye2022experimental}. Existing sensor-based alternatives, including inertial, magnetic, and acoustic systems, partially address these vulnerabilities yet suffer from calibration inaccuracies, limited precision, and environmental disruptions. Specialized signal processing techniques, though beneficial, introduce additional complexity and hardware demands, limiting their widespread practical deployment.

To address these challenges, researchers have begun to explore vision-based solutions that leverage multi-directional imagery for robust spatial awareness in urban environments with limited signal coverage --- a common condition in high-rise city centers where last-mile delivery and emergency logistics operations are most needed. For instance, recent work has demonstrated how fiducial markers can be incorporated into stereo visual-inertial odometry to achieve sub-meter accuracy under complex scenarios where conventional signal coverage is unreliable \cite{Wang2023UAV}. Likewise, collaborative visual-inertial simultaneous localization and mapping (SLAM) frameworks have emerged, showing that multiple mobile agents can share and fuse camera data to build dense 3D maps when navigating large-scale urban areas \cite{Zhang2022CVIDS}.

At the same time, the surge of edge computing offers new opportunities to offload computationally intensive tasks, particularly in situations where mobile endpoints lack the onboard resources to process large-scale imagery in real time \cite{fang2024pacp,10835069}. Rather than streaming raw or merely compressed video, task-oriented communication strategies focus on transmitting only the features most salient for specific tasks such as visual positioning \cite{Kang2022Task}. Such approaches effectively reduce bandwidth consumption and can adapt to changing network conditions. Similarly, learned feature compression techniques in split computing architectures have been shown to improve accuracy and reduce overhead, especially for tasks like object detection and classification \cite{Yuan2024Split}.

Motivated by these observations, we proactively outsource computation tasks \cite{10536071} while considering communication bottlenecks \cite{fang2025ton}, and propose an edge-assisted, multi-camera visual positioning system utilizing a novel {\underline{\textbf{O}}}rthogonally-constrained {\underline{\textbf{V}}}ariational {\underline{\textbf{I}}}nformation {\underline{\textbf{B}}}ottleneck encoder (O-VIB) enhanced with automatic relevance determination (ARD). O-VIB encourages the automatic collapse of irrelevant feature dimensions and ensures minimal redundancy across latent feature dimensions via orthogonality constraints. This method significantly improves transmission efficiency, directly aligning with task-oriented communication objectives for visual place recognition tasks.

The main contributions of our work are summarized as follows:
\begin{itemize}
\item We propose O-VIB, an ARD-enhanced VIB encoder with orthogonality constraints, effectively compressing multi-view visual features to significantly reduce transmission overhead without compromising positioning accuracy.
\item We release a large-scale, multi-camera urban positioning dataset comprising 357,690 frames with RGB, semantic segmentation, and depth data across multiple urban environments, targeting scenarios with limited signal availability.
\item Experimental validation on a physical testbed confirms that our O-VIB framework delivers high-precision positioning with minimal bandwidth, affirming its potential for latency-sensitive delivery and logistics applications in mobile edge networks.
\end{itemize}

\section{System Overview}

\subsection{System Architecture and Problem Formulation}

\begin{figure}[t]
  \centering
  \includegraphics[width=0.4\textwidth]{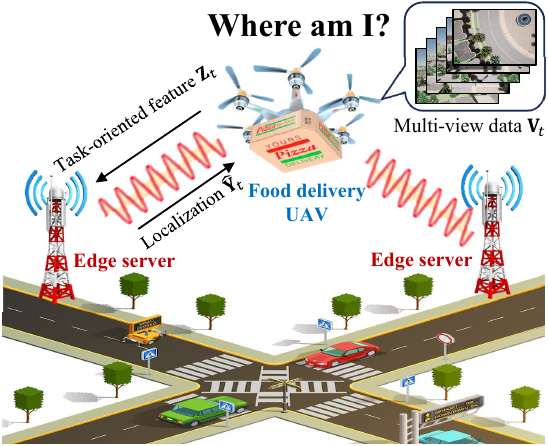}
  \caption{The system model of Edge-aerial collaboration.}
  \label{fig:architecture}
  \vspace{-3mm}
\end{figure}

Consider a UAV-edge collaborative system operating in a GPS-denied urban environment, as illustrated in Fig.~\ref{fig:architecture}. We denote by $\mathcal{U}$ the multi-camera UAV that captures multi-directional view $\mathcal{V} = \{v_1, v_2, \ldots, v_M\}$ where $M=5$ indicates our five-camera configuration (Front, Back, Left, Right, Down). The edge server, denoted by $\mathcal{E}$, maintains a geo-tagged feature database $\mathcal{D} = \{(f_i, l_i)\}_{i=1}^N$, where $f_i$ represents the visual features and $l_i = (x_i, y_i, z_i)$ denotes the corresponding 3D position. The UAV captures multi-view images at time step $t$, represented as $\mathbf{V}_t = \{V_t^{(1)}, V_t^{(2)}, \ldots, V_t^{(M)}\}$. For each view $m \in \{1,2,\ldots,M\}$, a feature extractor $\Phi(\cdot)$ generates high-dimensional features $\mathbf{X}_t^{(m)} = \Phi(V_t^{(m)}) \in \mathbb{R}^d$, where $d=512$ in our implementation. The concatenated multi-view features are denoted as $\mathbf{X}_t = [\mathbf{X}_t^{(1)}, \mathbf{X}_t^{(2)}, \ldots, \mathbf{X}_t^{(M)}] \in \mathbb{R}^{M \times d}$. Our objective is to accurately localize the UAV while minimizing communication overhead. Formally, we aim to solve $\min_{\Theta}\;
   \mathbb{E}\!\bigl[\|\posest{t}-\posgt{t}\|_{2}^{2}\bigr],$
where $\mathcal{C}(\mathbf Z_t)\le C_{\max},$ and $\posest{t}$ represents the estimated UAV position, $\posgt{t}$ is the ground truth position, $\Theta$ denotes the trainable parameters of our framework, $\mathbf{Z}_t$ is the compressed representation transmitted from UAV to edge servers, $\mathcal{C}(\cdot)$ is the communication cost function, and $C_{\max}$ is the maximum allowable communication bandwidth. The extracted features $\mathbf x$, the encoded features $\mathbf z$, and the position estimation $\mathbf y$ are instantiated by random variables $\mathbf{X}_t$, $\mathbf{Z}_t$ and $\mathbf{Y}_t$, respectively.

\subsection{Wireless Communication Model}

For the wireless link between the UAV and an edge server, we adopt the Shannon capacity model. Thus, the achievable data rate $R$ (in bits/s) can be expressed as $R = B \log_2 \left(1 + \frac{P \cdot g}{N_0 \cdot B} \right),$
where $B$ is the channel bandwidth, $P$ is the transmit power, $g$ is the channel gain incorporating path loss, shadowing, and fading effects, while $N_0$ is the noise power spectral density. The channel gain $g$ is modeled as $g = g_0 \cdot \left(\frac{d_0}{d}\right)^\alpha \cdot 10^{\frac{\xi}{10}} \cdot |h|^2$, where $g_0$ is the reference channel gain at distance $d_0$, $d$ is the distance between UAV and edge server, $\alpha$ is the path loss exponent (typically 2-4 in urban environments), $\xi \sim \mathcal{N}(0, \sigma^2)$ represents log-normal shadowing with standard deviation $\sigma$, and $h$ accounts for small-scale fading. The transmission delay for sending compressed features $\mathbf{Z}_t$ is given by $\tau = \frac{|\mathbf{Z}_t|}{R}$, where $|\mathbf{Z}_t|$ denotes the bit-size of the compressed representation.

\subsection{Multi-View Visual Localization Pipeline}

Our multi-view visual localization pipeline comprises three main components: Feature Extraction, Task-Oriented Feature Compression, and Edge-Based Position Inference.

In the \textbf{Feature Extraction} stage, each camera view is processed through a CLIP-based visual encoder $\Phi(\cdot)$ to extract discriminative features. For every view $m \in \{1,2,\ldots,M\}$ at time $t$, the encoded feature is computed as $\mathbf{X}_t^{(m)} = \Phi(V_t^{(m)}).$

As shown in Fig. \ref{fig:encoder}, during the \textbf{Task-Oriented Feature Compression} phase, the extracted multi-view features $\mathbf{X}_t$ are compressed into a task-relevant representation $\mathbf{Z}_t$ using a VIB-based encoder with orthogonality constraints, i.e.
\begin{equation}
\mathbf{Z}_t = \mathcal{E}(\mathbf{X}_t; \Theta_E),
\end{equation}
where $\mathcal{E}(\cdot)$ denotes our proposed encoding function parameterized by $\Theta_E$. As shown in Fig. \ref{fig:decoder}, in the \textbf{Edge-Based Position Inference} stage, the compressed representation is transmitted to the edge server, which applies multi-view attention fusion to estimate the UAV’s position
\begin{equation}
\posest{t} = \mathcal{F}(\mathbf{Z}_t; \Theta_F, \mathcal{D}).
\end{equation}
In this equation, $\mathcal{F}(\cdot)$ represents the fusion and localization function with parameters $\Theta_F$, and $\mathcal{D}$ refers to a geo-tagged database used for querying position information. The position estimation integrates a hybrid method that combines direct regression and retrieval-based inference
\begin{equation}
\posest{t} = \eta \cdot \posest{t}^{reg} + (1-\eta) \cdot \posest{t}^{ret},
\end{equation}
where $\posest{t}^{reg}$ is the regressed position, $\posest{t}^{ret}$ is the retrieved position from the database, and $\eta \in [0,1]$ is an adaptive weight that balances the two estimates based on confidence scores. It is noted that the above end-to-end pipeline is designed to optimize the trade-off between localization accuracy and communication efficiency, thereby enabling precise UAV navigation in GPS-denied environments with constrained wireless bandwidth.


\section{Methodology}\label{sec: method}

\begin{figure*}[t]
    \centering
    
    \begin{minipage}[t]{0.46\textwidth}
        \centering
        \includegraphics[width=\linewidth]{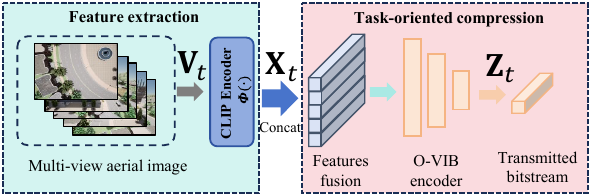}
        \caption{Feature-extraction and task-oriented compression pipeline executed on board the UAV.}
        \label{fig:encoder}
    \end{minipage}
    \hfill
    \begin{minipage}[t]{0.46\textwidth}
        \centering
        \includegraphics[width=\linewidth]{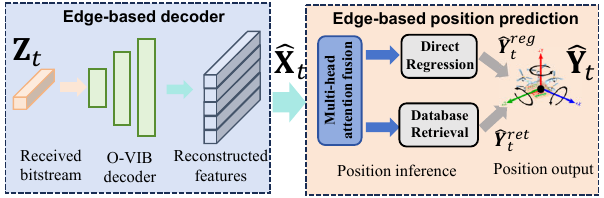}
        \caption{Edge-side decoding and position-prediction pipeline running on RSU servers.}
        \label{fig:decoder}
    \end{minipage}
    
    \vspace{-6mm}
\end{figure*}

\subsection{Task-Oriented Feature Extraction}

To enable discriminative localization under limited bandwidth, we employ a CLIP-based vision backbone for robust multi-view feature extraction. 
Each image $V_t^{(m)}$ is processed via a shared feature extractor $\Phi(\cdot)$ implemented using the CLIP Vision Transformer (ViT-B/32), pretrained on large-scale natural image-text pairs. The encoder $\Phi(\cdot)$ first applies a learned preprocessing function $\psi(\cdot)$ that resizes, normalizes, and tokenizes the image into a sequence of visual patches. The feature encoding is obtained as:
\begin{equation}
\mathbf{X}_t^{(m)} = \Phi(V_t^{(m)}) = f_{\text{CLIP}}\left(\psi(V_t^{(m)}); \theta_\Phi \right) \in \mathbb{R}^{d},
\end{equation}
where $f_{\text{CLIP}}(\cdot)$ denotes the CLIP image encoder, parameterized by $\theta_\Phi$, while $d = 512$ is the dimensionality of the output embedding space for ViT-B/32. We normalize the extracted features to lie on the unit hypersphere to improve numerical stability and facilitate cosine similarity-based downstream retrieval, which yields $\tilde{\mathbf{X}}_t^{(m)} = \frac{\mathbf{X}_t^{(m)}}{\|\mathbf{X}_t^{(m)}\|_2}, \forall m \in \{1, \ldots, M\}.$ Besides, the final multi-view feature tensor is constructed by concatenating view-wise embeddings $\mathbf{X}_t = \left[\tilde{\mathbf{X}}_t^{(1)}; \tilde{\mathbf{X}}_t^{(2)}; \ldots; \tilde{\mathbf{X}}_t^{(M)}\right] \in \mathbb{R}^{M \times d}$. This high-dimensional, view-aligned descriptor $\mathbf{X}_t$ captures a rich, panoramic representation of the UAV’s surroundings.

\subsection{Task–Oriented Feature Compression}\label{sec:compression}

The high-dimensional multi-view descriptor $\mathbf{X}_t \in \mathbb{R}^{M \times d}$ provides comprehensive visual information about the UAV's surroundings. However, due to stringent communication constraints in UAV-edge collaborative systems, it is necessary to compress this descriptor into a compact task-relevant representation\cite{fang2025ton}. The IB principle provides an effective theoretical framework for addressing this challenge. Specifically, IB seeks an optimal stochastic encoder $q_{\phi}(\mathbf{z}|\mathbf{x})$ that generates a latent representation $\mathbf{z}$ by achieving two conflicting goals: minimizing the mutual information $I(\mathbf{x};\mathbf{z})$ to ensure compactness, while maximizing the mutual information $I(\mathbf{z};\mathbf{y})$ to retain task-relevant information about the UAV's position $\mathbf{y}$. The IB optimization problem can thus be formulated as
\begin{equation} \label{OP:IB}
  \min_{\phi}\; 
  \underbrace{\beta\,I(\mathbf x;\mathbf z)}_{\text{Transmission}}-
  \underbrace{I(\mathbf z;\mathbf y)}_{\text{Accuracy}},
\end{equation}
where the non-negative hyperparameter $\beta$ controls the trade-off between feature compression (transmission efficiency) and localization accuracy.
\subsubsection{Why Automatic Relevance Determination (ARD)?}

The optimization problem in \eqref{OP:IB} already performs task-oriented
\emph{rate–relevance} trade-off, yet it counts every latent dimension
equally.  In practice, many coordinates are dispensable.
We impose an \emph{ARD sparsifier} by choosing for each $z_i$ the
\emph{log-uniform} prior
$
  p(z_i)\propto |z_i|^{-1}
$\cite{molchanov2017variational}.
Its heavy tail is scale-invariant and assigns virtually
\emph{no mass near zero}, therefore encouraging uninformative
coordinates to collapse automatically. Let $\mathbf x\in\mathbb R^{M\times d}$ be the concatenated
view features of one frame.
The encoder is a diagonal Gaussian, which is formulated as
$q_{\phi}(\mathbf z\mid\mathbf x)=
  \mathcal N\!\bigl(\boldsymbol\mu_{\phi}(\mathbf x),
          \operatorname{diag}\boldsymbol\sigma^{2}_{\phi}(\mathbf x)\bigr),$
          where $\mathbf z=\boldsymbol\mu+\boldsymbol\sigma\odot\boldsymbol\epsilon,
  \boldsymbol\epsilon\sim\mathcal N(\mathbf 0,\mathbf I).$
The KL divergence between a univariate Gaussian posterior and the log-uniform prior admits an accurate analytic fit
\vspace{-2pt}
\begin{small}
\begin{equation}
\begin{aligned}
  \mathcal D_{\textsc{ard}}\!
  \bigl(\boldsymbol\mu_{\phi},\log\boldsymbol\sigma^{2}_{\phi}\bigr)
  &= \frac1B\sum_{i=1}^{k}\!
  \Bigl[
      k_{1}\,\sigma\!\bigl(k_{2}+k_{3}\log\alpha_{i}\bigr)\\
      &-{\frac{1}{2}}\log\!\bigl(1+e^{-\log\alpha_{i}}\bigr)
  \Bigr] \\
  &\approx \mathrm{KL}\!\bigl(q_{\phi}(\mathbf z\mid\mathbf x)\,\|\,p(\mathbf z)\bigr)
  ,
  \label{eq:ardkl}
\end{aligned}
\end{equation}
\end{small}
where $\alpha_i:=\sigma_i^{2}/\mu_i^{2}$, and the predetermined coefficients are $(k_{1},k_{2},k_{3})=(0.63576,\,1.87320,\,1.48695)$. $B$ represents the minibatch size. According to \textbf{Theorem} \ref{thm:ib_ard}, the traditional IB objective~\eqref{OP:IB} is upper-bounded by the ARD–regularized variational objective in (\ref{eq:vib_obj}).




\begin{lemma}
\label{prop:chain_rule}%
Let $q_{\phi}(\mathbf z\mid\mathbf x)$ be any encoder and let
$p(\mathbf z)$ be an arbitrary prior.  Define the variational mutual
information
$I_{q_{\phi}}(\mathbf x;\mathbf z)
 :=\mathrm{KL}\!\bigl(q_{\phi}(\mathbf x,\mathbf z)\,
      \Vert\,q_{\phi}(\mathbf x)\,q_{\phi}(\mathbf z)\bigr)$
and the marginal
$q_{\phi}(\mathbf z)=\int q_{\phi}(\mathbf z\mid\mathbf x)p(\mathbf x)d\mathbf x$.
Then
\begin{align}
 I_{q_{\phi}}(\mathbf x;\mathbf z)
 &=\mathbb E_{\mathbf x}\!
    \Bigl[\mathrm{KL}\bigl(q_{\phi}(\mathbf z\mid\mathbf x)\,\Vert\,p(\mathbf z)\bigr)\Bigr]
   -\mathrm{KL}\bigl(q_{\phi}(\mathbf z)\,\Vert\,p(\mathbf z)\bigr)
   \label{eq:chain_rule_exact}\\
 &\le \mathbb E_{\mathbf x}\!
    \Bigl[\mathrm{KL}\bigl(q_{\phi}(\mathbf z\mid\mathbf x)\,\Vert\,p(\mathbf z)\bigr)\Bigr].
   \nonumber
\end{align}
If $p(\mathbf z)$ is chosen coordinate-wise log-uniform
($p(z_i)\propto|z_i|^{-1}$) and $q_{\phi}$ is diagonal Gaussian, the
inner KL admits the accurate analytic fit
$\mathcal D_{\textsc{ard}}\bigl(
     \boldsymbol\mu_{\!\phi}(\mathbf x),
     \log\boldsymbol\sigma^{2}_{\!\phi}(\mathbf x)\bigr)$
of~\cite{molchanov2017variational}, so that
\begin{equation}
  I_{q_{\phi}}(\mathbf x;\mathbf z)
  \;\le\;
  \mathbb E_{\mathbf x}\bigl[\mathcal D_{\textsc{ard}}(\mathbf x)\bigr].
  \label{eq:rate_upper}
\end{equation}
\end{lemma}

\begin{proof}
Eq.~\eqref{eq:chain_rule_exact} follows from the classical chain
rule for KL divergence; dropping the second (non-negative) term yields
the inequality. \(\mathcal D_{\textsc{ard}}(\mu_i,\log\sigma_i^2)\approx\mathrm{KL}\!\bigl(q_{\phi}(\mathbf z\mid\mathbf x)\,\|\,p(\mathbf z)\bigr)\), which Molchanov \emph{et al.} show has a maximum absolute error below \(10^{-3}\) for \(\alpha_i\in[10^{-4},10^4]\)\cite{molchanov2017variational}.
Replacing the conditional KL by its
$\mathrm{ARD}$ fit gives Ineq. \eqref{eq:rate_upper}.{\hfill $\blacksquare$\par}
\end{proof}

\begin{lemma}\label{prop:izy_bound}
For any decoder \( p_{\theta}(\mathbf{y}|\mathbf{z}) \) and joint distribution \( p(\mathbf{z},\mathbf{y}) \), the mutual information between latent representation \(\mathbf{z}\) and task variable \(\mathbf{y}\) is lower-bounded by
\begin{equation}\label{eq:izy_bound}
  I(\mathbf{z};\mathbf{y}) \geq \mathbb{E}_{\mathbf{z},\mathbf{y}}\left[\log p_{\theta}(\mathbf{y}|\mathbf{z})\right] - \mathrm{H}(\mathbf{y}),
\end{equation}
where \(\mathrm{H}(\mathbf{y}) = -\mathbb{E}_{\mathbf{y}}\left[\log p(\mathbf{y})\right]\) is the entropy of \(\mathbf{y}\), a constant independent of model parameters \((\phi,\theta)\).
\end{lemma}

\begin{proof}
According to definition, we have $I(\mathbf{z};\mathbf{y}) = \mathrm{H}(\mathbf{y}) - \mathrm{H}(\mathbf{y}|\mathbf{z})$. Since the KL divergence is always non-negative, we have $-\mathrm{H}(\mathbf{y}|\mathbf{z}) \geq \mathbb{E}_{\mathbf{z},\mathbf{y}}\left[\log p_{\theta}(\mathbf{y}|\mathbf{z})\right]$.
Combining the above relations, we obtain $I(\mathbf{z};\mathbf{y}) \geq \mathbb{E}_{\mathbf{z},\mathbf{y}}\left[\log p_{\theta}(\mathbf{y}|\mathbf{z})\right] - \mathrm{H}(\mathbf{y})$.
\hfill\(\blacksquare\)
\end{proof}


\begin{theorem}\label{thm:ib_ard}
The traditional IB objective~\eqref{OP:IB}
is upper-bounded by the \emph{ARD–regularized}
variational objective
\begin{equation}
  \min_{\phi}\;
  \beta\,\mathbb E_{\mathbf x}\bigl[\mathcal D_{\textsc{ard}}(\mathbf x)\bigr]
  -\mathbb E_{\mathbf z,\mathbf y}
      \bigl[\log p_{\theta}(\mathbf y\mid\mathbf z)\bigr],
  \label{eq:vib_obj}
\end{equation}
which is tractable, differentiable, and automatically
prunes uninformative latent coordinates by driving their contribution to near-zero.
\end{theorem}

\textbf{Theorem} \ref{thm:ib_ard} can be proven by combining \textbf{Lemmas}~\ref{prop:chain_rule} and \ref{prop:izy_bound}. Eq.~\eqref{eq:vib_obj} is what we optimize in practice. More explicitly, the
first term is computed with~\eqref{eq:ardkl},
while the second term is tightened by a variational decoder
$p_{\theta}(\mathbf y\mid\mathbf z)$ as in
\textbf{Lemma}~\ref{prop:izy_bound}.

\subsubsection{Orthogonality Under the IB Objective}\label{subsubsec:orthogonal_ib}

In our VIB-based encoder design, we compress the multi-view feature $\mathbf{x}\in \mathbb{R}^{M\times d}$ into a low-dimensional latent representation $\mathbf{z} \in \mathbb{R}^k$. Following the IB principle, we maintain a small mutual information $\beta I(\mathbf{x};\mathbf{z})$ to limit bandwidth usage while retaining sufficient information about $\mathbf{x}$ for accurate prediction of $\mathbf{y}$. To optimize the utilization of the limited information budget, we impose approximate row-orthogonality on the encoder's weight matrix $\mathbf{W}$. The proposition \ref{proposition:vib_orthogonal} shows that if $\mathbf{W}\mathbf{W}^\top$ is close to an identity matrix, the variance of each latent coordinate remains close to the average $\frac{1}{k}\operatorname{Tr}(\mathbf{W}\Sigma_x\mathbf{W}^\top)$. Consequently, no latent dimension collapses to near-zero variance, thereby avoiding redundancy in the latent representation.

\begin{proposition}\label{proposition:vib_orthogonal}
Let \(\mathbf{W}\in\mathbb{R}^{k\times d}\) denote the weight matrix of the Variational Information Bottleneck (VIB) encoder layer. Assume the approximate orthogonality condition
\[
\mathbf{W}\mathbf{W}^\top = \mathbf{I}_k + \boldsymbol{\Delta},
\]
where \(\mathbf{I}_k\) denotes the \(k\times k\) identity matrix and \(\boldsymbol{\Delta}\) is a symmetric perturbation matrix satisfying \(\|\boldsymbol{\Delta}\| \le \varepsilon\) for a small \(\varepsilon>0\). Let \(\mathbf{x}\in\mathbb{R}^d\) be an input vector with covariance \(\Sigma_x = \mathrm{Cov}(\mathbf{x})\). Define the latent representation \(\mathbf{z} = \mathbf{W}\mathbf{x}\), and let \(a_i = \mathbf{w}_i \Sigma_x \mathbf{w}_i^\top\) represent the variance along the \(i\)th latent dimension, where \(\mathbf{w}_i\) is the \(i\)th row of \(\mathbf{W}\). Define the average variance \(T = \frac{1}{k}\sum_{i=1}^k a_i = \frac{1}{k}\operatorname{Tr}(\mathbf{W}\Sigma_x\mathbf{W}^\top)\). Then, the minimum variance satisfies
\[
\min_{1\le i\le k} a_i \ge T - C\varepsilon,
\]
where \(C>0\) is a constant dependent on \(\|\Sigma_x\|\).
\end{proposition}

\begin{Proof}
We begin by observing the latent covariance \(\Sigma_z = \mathrm{Cov}(\mathbf{z}) = \mathbf{W}\Sigma_x\mathbf{W}^\top\). Hence, the total variance is given by
\[
\operatorname{Tr}(\Sigma_z) = \sum_{i=1}^k \mathbf{w}_i\Sigma_x\mathbf{w}_i^\top = \sum_{i=1}^k a_i,
\]
and thus \(T = \frac{1}{k}\operatorname{Tr}(\Sigma_z)\). Under ideal orthogonality (\(\mathbf{W}\mathbf{W}^\top = \mathbf{I}_k\)), each dimension captures an equal variance \(T\). However, approximate orthogonality in practice is modeled by decomposing \(\mathbf{W}\)
\[
\mathbf{W} = \mathbf{W}_0 + \Delta \mathbf{W},\quad \text{with}\quad \mathbf{W}_0\mathbf{W}_0^\top = \mathbf{I}_k,
\]
and perturbation \(\Delta\mathbf{W}\) satisfying \(\mathbf{W}\mathbf{W}^\top = \mathbf{I}_k + \boldsymbol{\Delta}\) and \(\|\boldsymbol{\Delta}\| \le \varepsilon\). Considering \(\mathbf{w}_i = \mathbf{w}_{0,i}+\Delta\mathbf{w}_i\), we obtain:
\[
a_i = (\mathbf{w}_{0,i}+\Delta \mathbf{w}_i)\Sigma_x(\mathbf{w}_{0,i}+\Delta \mathbf{w}_i)^\top.
\]
Expanding this expression yields:
\[
a_i = \underbrace{\mathbf{w}_{0,i}\Sigma_x\mathbf{w}_{0,i}^\top}_{\text{Constant} \ T} + \underbrace{2\Delta \mathbf{w}_i\Sigma_x\mathbf{w}_{0,i}^\top}_{\text{First-order term}} + \underbrace{\Delta \mathbf{w}_i\Sigma_x\Delta \mathbf{w}_i^\top}_{\text{Second-order term}}.
\]

Applying the Cauchy–Schwarz inequality and noting $\lVert\mathbf w_{0,i}\rVert_2=1$, the perturbation terms satisfy:
\[
|2\Delta \mathbf{w}_i\Sigma_x\mathbf{w}_{0,i}^\top| \le 2\|\Delta \mathbf{w}_i\|\|\Sigma_x\|,
\]
and similarly, we have $|\Delta \mathbf{w}_i\Sigma_x\Delta \mathbf{w}_i^\top| \le \|\Delta \mathbf{w}_i\|^2\|\Sigma_x\|$.
Since \(\|\Delta \mathbf{w}_i\| \le \varepsilon\), it follows that $|a_i - T| \le 2\varepsilon\|\Sigma_x\| + \varepsilon^2\|\Sigma_x\|$. For sufficiently small \(\varepsilon\), we consolidate terms and define a constant \(C\) depending on \(\|\Sigma_x\|\), yielding $|a_i - T| \le C\varepsilon,$ and consequently we can prove $\min_{1\le i\le k} a_i \ge T - C\varepsilon$.
\(\hfill\blacksquare\)
\end{Proof}

According to \textbf{Proposition} \ref{proposition:vib_orthogonal}, the approximate orthogonality of the encoder weights ensures that each latent dimension retains significant variance, avoiding collapsed dimensions. Under tight information bottleneck constraints, this property ensures each latent dimension effectively contributes to preserving relevant information, optimizing the latent representation \(\mathbf{z}\) with respect to the target variable \(\mathbf{y}\). This behavior maximizes the mutual information \(I(\mathbf{z};\mathbf{y})\) subject to channel capacity constraints, enhancing the efficiency and accuracy of task-oriented data compression.

\subsubsection{Joint Encoding for Reducing Multi-View Redundancy}\label{sec:joint_enc}

Rather than compressing each camera stream in isolation, we concatenate the \(M\) view-wise embeddings into a single vector and pass it through one VIB–ARD encoder.  
This strategy exploits inter–view correlations so that the latent code stores only complementary information.



\subsubsection{Overall Training Objective}

Combining reconstruction fidelity, localization accuracy, ARD-based rate control, and orthogonality regularisation yields the composite loss
\begin{equation}
\begin{aligned}
\mathcal L(\phi)= &\;
\underbrace{\mathbb E\bigl[\lVert\mathbf x-\hat{\mathbf x}\rVert^{2}\bigr]}_{\text{Reconstruction}}
+ \alpha\underbrace{\mathbb E\bigl[\lVert{\mathbf y-\hat{\mathbf y}}\rVert^{2}\bigr]}_{\text{Localisation}} \\[2pt]
&+ \beta\underbrace{\mathbb E_{\mathbf x}\!\bigl[\mathcal D_{\textsc{ard}}(\mathbf x)\bigr]}_{\text{Information bottleneck}}
+ \gamma\underbrace{\lVert\mathbf W\mathbf W^{\top}-\mathbf I\rVert^{2}_{F}}_{\text{Orthogonality}},
\end{aligned}
\end{equation}
where \(\mathbf x\) and \(\hat{\mathbf x}\) are respectively the input and reconstructed multi-view features,  
\(\mathbf y\) and \(\hat{\mathbf y}\) are respectively the ground-truth and predicted UAV positions. Furthermore, 
\(\mathcal D_{\textsc{ard}}\) is the closed-form KL term in~\eqref{eq:ardkl} that promotes sparsity of the latent code,  
while \(\mathbf W\) denotes the weight matrix of the encoder projection.  
The coefficients \(\alpha,\beta,\gamma>0\) balance the four competing objectives.  
Orthogonality guarantees that every retained latent dimension remains informative, whereas the ARD penalty drives superfluous coordinates toward zero variance, thereby enabling hard pruning after training.

\section{Performance Evaluation}
\begin{figure*}[t]
    \centering

    \begin{minipage}[t]{0.41\textwidth}
        \centering
        \includegraphics[width=\linewidth]{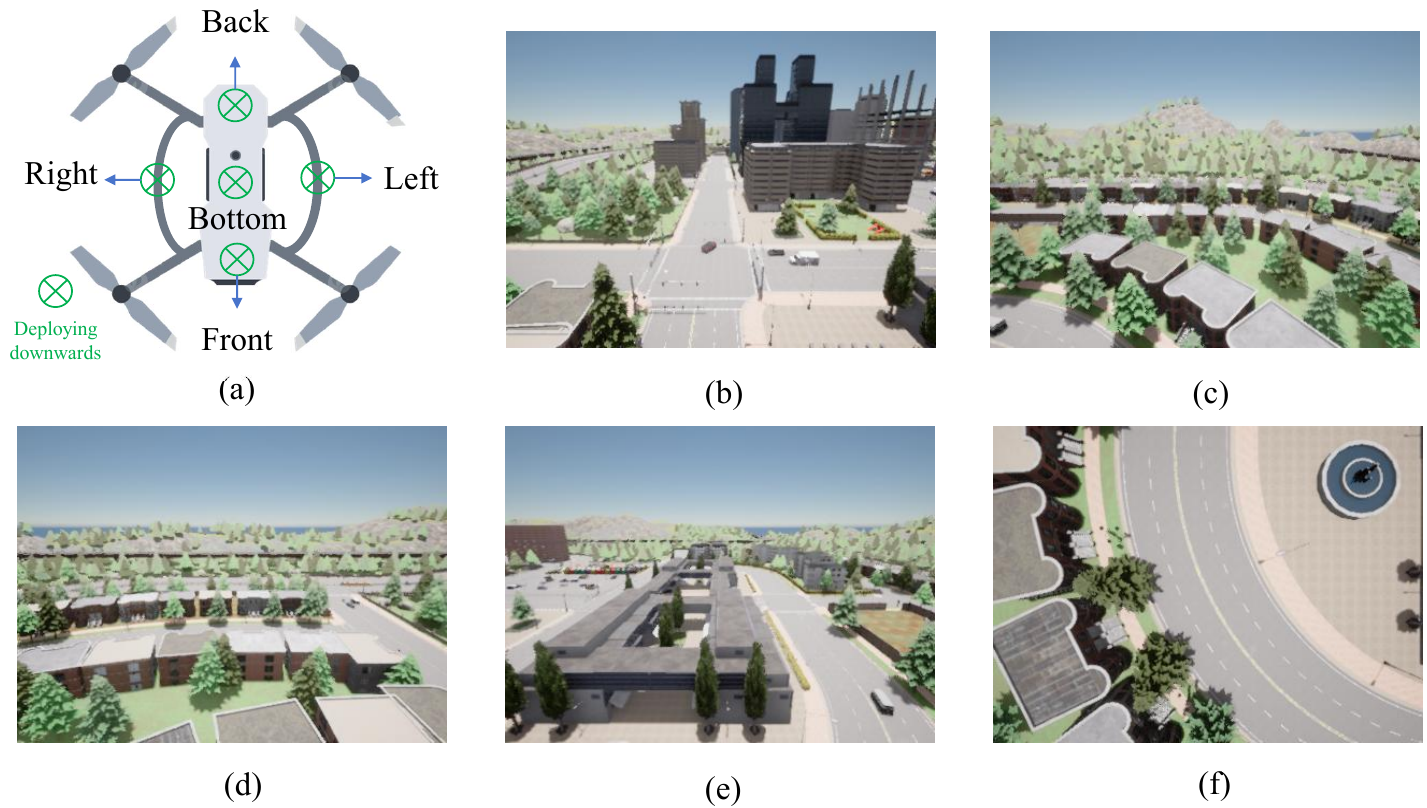}
        \caption{Multi-camera UAV perception system and corresponding visual observations.}
        \label{fig:drone-camera-views}
    \end{minipage}
    \hfill
    \begin{minipage}[t]{0.48\textwidth}
        \centering
        \includegraphics[width=\linewidth]{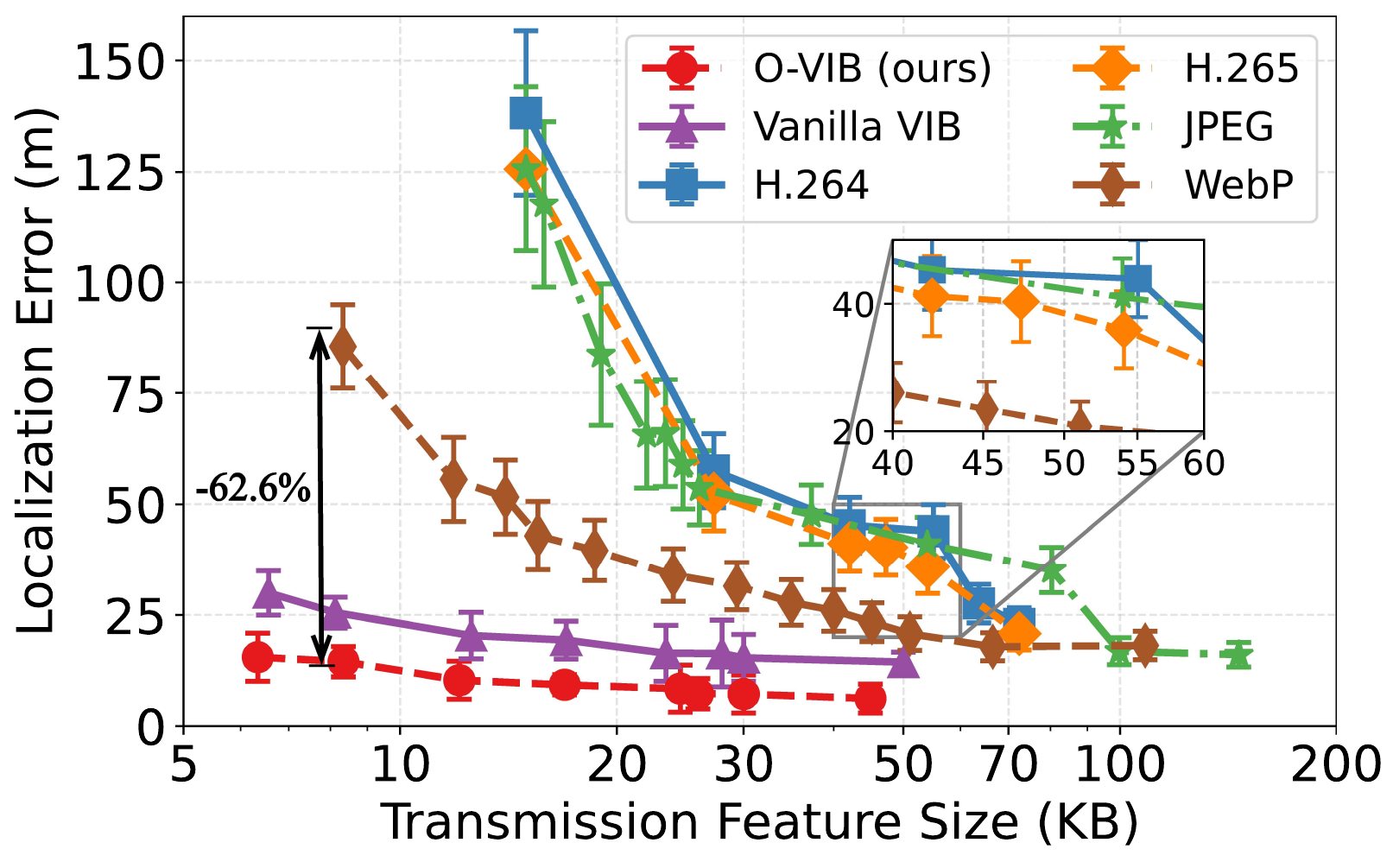}
  \caption{Transmission feature size vs localization error.}
  \label{fig:CM}
    \end{minipage}
\vspace{-3mm}
\end{figure*}

As shown in Fig.~\ref{fig:drone-camera-views}, we collect a new dataset for visual navigation of UAVs in the CARLA simulator that mimics GNSS-denied flight over eight representative maps in cities.  
A UAV flies at a constant height following waypoints aligned with the road, changing direction randomly. 
Five onboard cameras capture images from different angles and directions, recording RGB, semantic, and depth images at a 400 × 300 pixel resolution. A total of 357,690 multi-view frames are recorded, each labeled with the precise localization and rotation. The framework is deployed on real hardware devices to evaluate the computation and communication latency of our proposed methods. 
Each UAV carries a Jetson Orin NX 8GB that encodes five camera streams and transmits compressed features to the nearby roadside units (RSUs) through wireless channels (IEEE 802.11).  
Two classes of RSU are deployed: (i) \emph{Relay RSU}: Raspberry Pi 5 16 GB that forward data by Gigabit Ethernet to a cloud edge server when overloaded; (ii) \emph{Edge RSU}: Jetson Orin NX Super 16 GB that performs on-board inference.

Our primary metric is localization error (Euclidean distance to ground-truth pose).  
We additionally report the entropy–performance trade-off, i.e.\ how latent entropy after pruning correlates with accuracy, and measure end-to-end latency (UAV capture $\rightarrow$ position estimation in edge server). Moreover, we compare our orthogonal VIB encoder to five advanced codecs: vanilla VIB \cite{10480247}, JPEG \cite{wallace1992jpeg}, H.264 \cite{H264}, H.265/HEVC \cite{bossen2012hevc}, and WebP \cite{10605825}.  
All baselines are tuned to match our bitrate range for fair comparison.

Fig. \ref{fig:CM} shows localization accuracy as the UAV–RSU link budget (in KB/s) is varied. Each scheme’s parameter set (e.g.\ latent dimension for VIB; quality factor for JPEG/H.264/H.265/WebP) is swept to identify Pareto‐optimal points—minimizing error per KB/s. Features are encoded onboard, sent over IEEE 802.11, decoded at the RSU, and matched to a geo-tagged database; the nearest neighbor index yields the pose estimate. When network throughputs are above 50 KB/s, all methods converge to a mean error of 10 m. When the bottleneck falls below 10 KB/s, O-VIB degrades most gracefully: at 8 KB/s it still achieves less than 10 m error, representing a 42.1\% reduction versus vanilla VIB and a 62.6\% reduction versus WebP. Embedding orthogonality thus prunes redundant latent dimensions while preserving task-critical information, making O-VIB far more robust under severe bandwidth constraints.

%

\begin{figure}[t]
  \centering
  \includegraphics[width=0.48\textwidth]{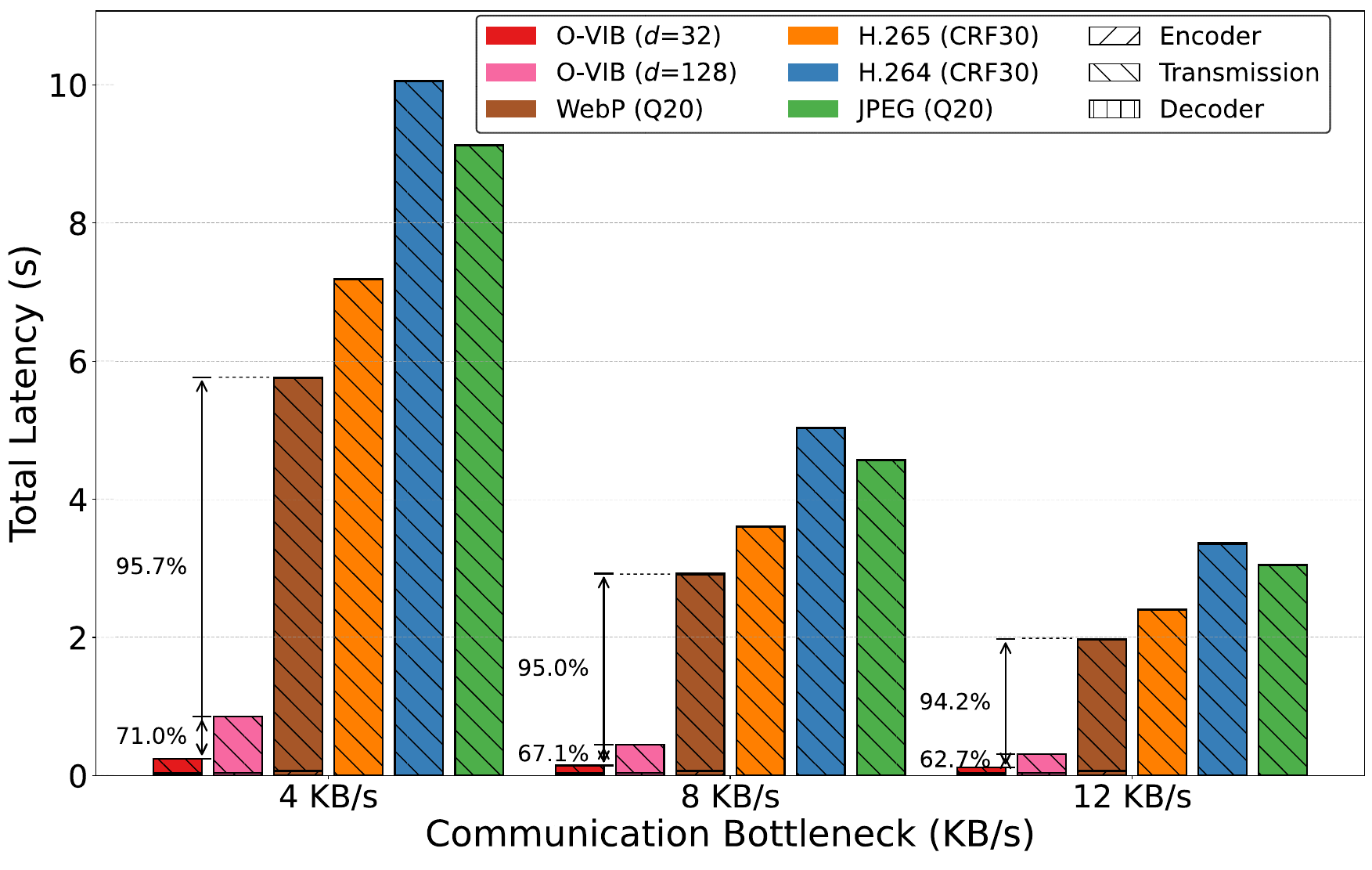}
  \caption{\color{black}{{Communication bottleneck vs latency.}}}
  \label{fig:Communication bottleneck vs latency}
\end{figure}

Fig.~\ref{fig:Communication bottleneck vs latency} shows end-to-end latency (encoding, transmission, decoding) under 4, 8, and 12~KB/s bottlenecks (poor channel conditions). At 4~KB/s, O-VIB achieves $0.24$~s ($d=32$) and $0.85$~s ($d=128$), while WebP, H.265, H.264, and JPEG incur $5.7$~s, $7.1$~s, $10.9$~s, and $9.1$~s. Compared to WebP, O-VIB reduces latency by $95.7\%$. At 8~KB/s, O-VIB drops to $0.15$~s ($d=32$) and $0.44$~s ($d=128$), while WebP remains at $2.9$~s, achieving a $95.0\%$ reduction. At 12~KB/s, O-VIB further lowers latency to $114.0$~ms and $0.31$~s, compared to WebP's $1.9$~s, realizing a $94.2\%$ reduction. These results confirm that O-VIB maintains sub-second latency and achieves over an order of magnitude improvement under stringent bottlenecks.

\begin{figure}[t]
  \centering
  \subfigure[Localization Error vs $\beta$.]{
    \includegraphics[width=0.22\textwidth]{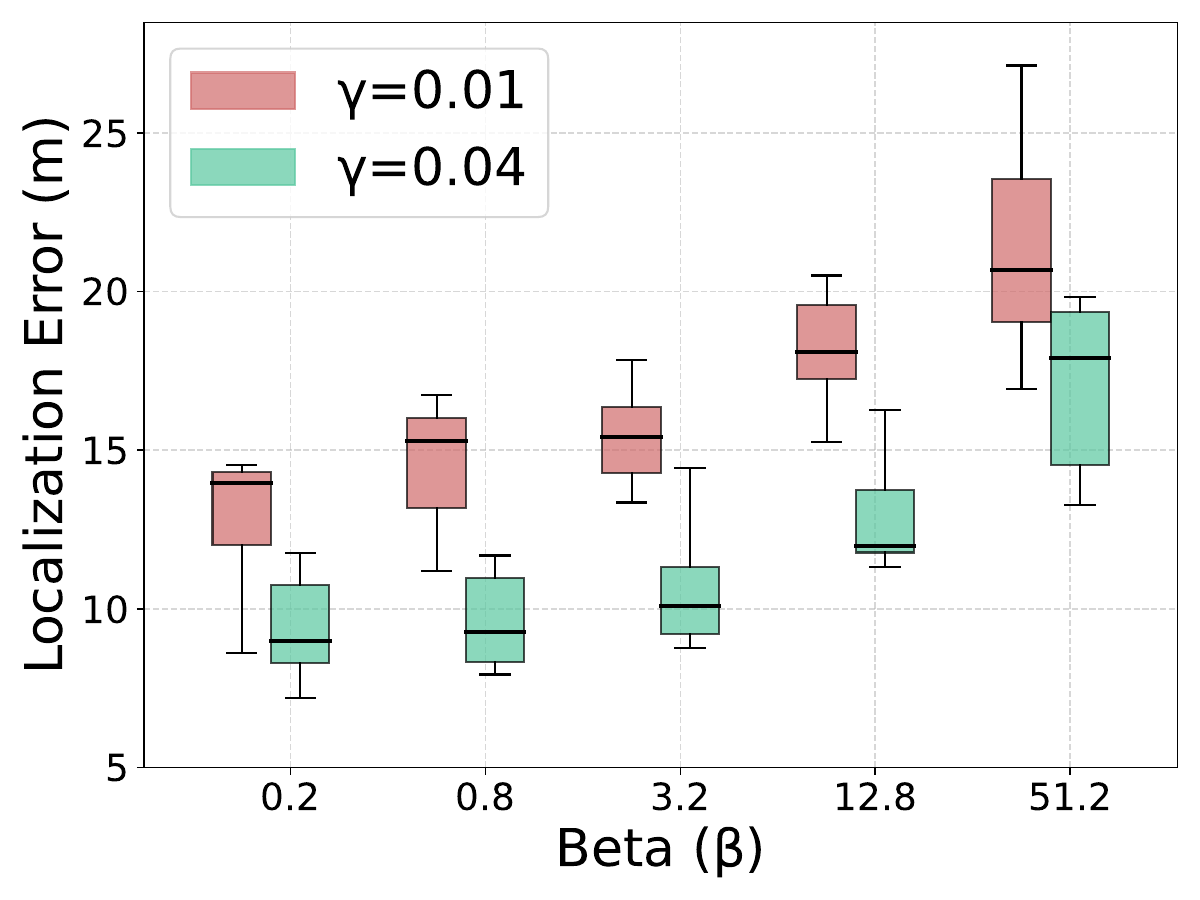}
  }
  \subfigure[{\color{black}{Latent Entropy vs $\beta$.}}]{
    \includegraphics[width=0.22\textwidth]{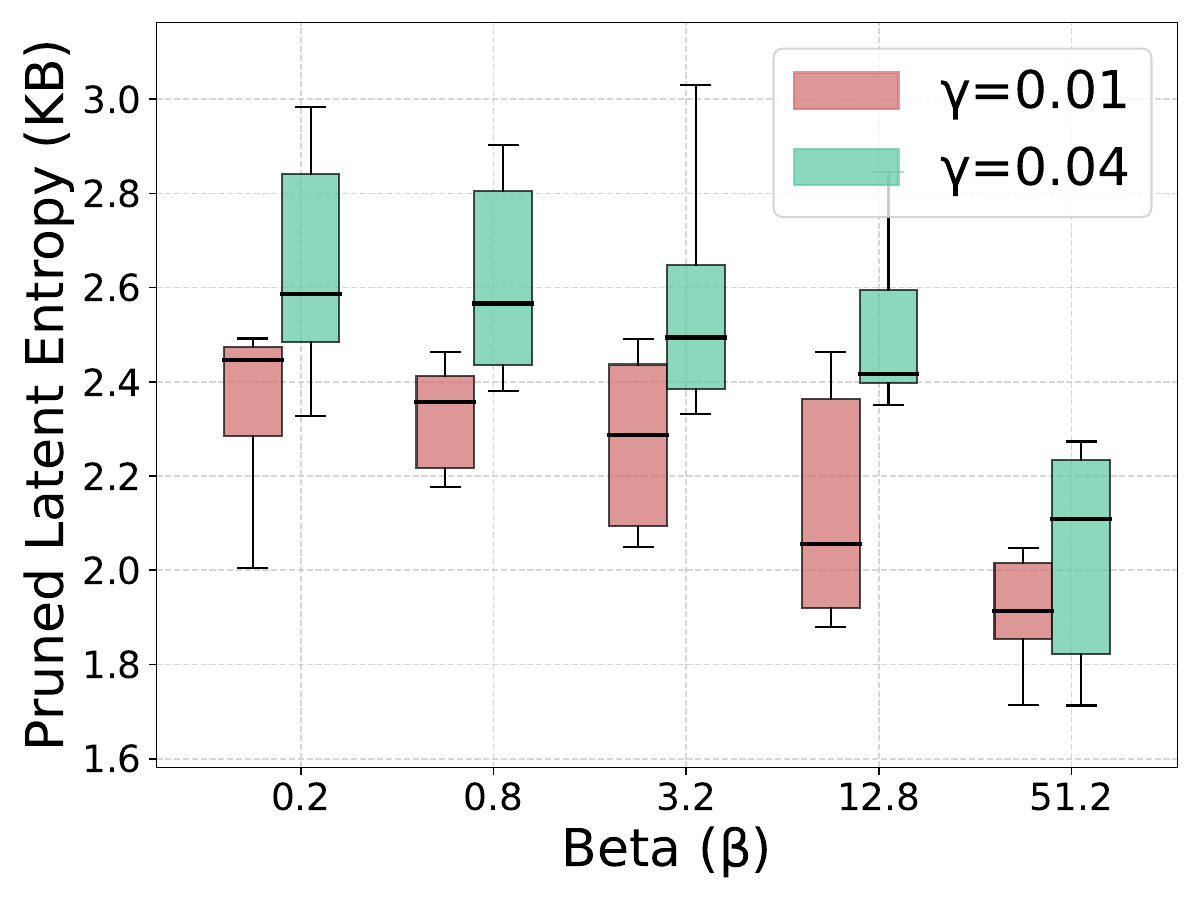}
  }
  \caption{\color{black}{{Localization error and latent entropy vs $\beta$.}}}
  \label{fig:loc-entropy}
  \vspace{-3mm}
\end{figure}

Fig. \ref{fig:loc-entropy} explores how the information‐bottleneck weight $\beta$ shapes the trade-off between compression and localization under two orthogonality strengths ($\gamma = 0.01, 0.04$). In Fig. \ref{fig:loc-entropy}(a), increasing $\beta$ steadily reduces pruned latent entropy, confirming that the ARD term drives superfluous dimensions toward zero variance. At the same time, localization error rises from about 12 m to over 20 m for $\gamma = 0.01$ and from  9 m to 19 m for $\gamma$ = 0.04, demonstrating the expected accuracy penalty of tighter compression. The $\gamma$ = 0.04 curves consistently lie below $\gamma = 0.01$ in Fig. \ref{fig:loc-entropy}(a) while exhibiting higher entropy in Fig. \ref{fig:loc-entropy}(b), validating that stronger orthogonality preserves task-critical information and yields better localization at equivalent rates.

\section{Conclusion}
In this paper, we have proposed a task‐oriented communication framework for visual navigation with edge–aerial collaboration for low altitude economy. Our contributions are twofold. First, we have developed a multi‐camera variational information bottleneck encoder augmented with an orthogonality constraint, which extracts ultra‐compact, task‐relevant features from five camera views. Second, we have deployed and evaluated the complete system on both a new CARLA‐derived dataset and a physical Jetson Orin NX/Raspberry Pi testbed, quantifying localization accuracy, latent‐entropy trade‐offs, and end‐to‐end latency. Extensive experiments have demonstrated that O-VIB maintains sub-10 m localization error at throughputs below 10 KB/s—reducing error by 42.1 \% versus vanilla VIB and 62.6 \% versus WebP—and achieves over three orders‐of‐magnitude lower latency than JPEG, H.264 and H.265. We will release our dataset and code to accelerate future research in task-oriented aerial communications.

\section{Acknowledgement}
This work of Y. Fang was supported in part by the Hong Kong SAR Government under the Global STEM Professorship and Research Talent Hub, the Hong Kong Jockey Club under the Hong Kong JC STEM Lab of Smart City (Ref.: 2023-0108). The work of S. Hu was supported in part by the Hong Kong Innovation and Technology Commission under InnoHK Project CIMDA. The work of Y. Deng was supported in part by the National Natural Science Foundation of China under Grant No. 62301300. 


%


\bibliographystyle{./IEEEtran}
\bibliography{IEEEabrv,ref}

\end{document}